\documentclass[11pt,a4paper]{article}
\usepackage[margin=24mm]{geometry}
\usepackage[T1]{fontenc}
\usepackage{lmodern}
\usepackage{amsmath,amssymb,amsthm}
\usepackage{graphicx,booktabs,longtable,array}
\usepackage[numbers,sort&compress]{natbib}
\usepackage{microtype}
\usepackage{caption}
\usepackage{float}
\usepackage[hidelinks]{hyperref}
\usepackage{adjustbox}
\usepackage{placeins}
\newtheorem{proposition}{Proposition}
\graphicspath{{figures/}}
\title{Power-of-Two Leaky Integrate-and-Fire Neurons for Event-Driven Vision--Language--Action Inference}
\author{Rui Liu \quad He Lv \quad Ziyao Zhang \quad Rui Ma\\Rongxiang Nie\thanks{Corresponding author: \texttt{nierongxiang@sivy.ltd}.}\\Sivy Limited, Hangzhou, China\\www.sivy.ltd}
\date{}
\begin{document}
\maketitle

\begin{abstract}

Vision--language--action (VLA) models have demonstrated strong capabilities in robotic manipulation, yet their reliance on dense, continuous-valued computation presents a fundamental challenge for energy-constrained deployment. Inspired by the sparse, event-driven signaling and temporal integration mechanisms of biological neurons, a promising alternative is to introduce spiking computation into VLA models. However, enabling such computation without restructuring pretrained policies or compromising closed-loop control remains an open challenge. To address this challenge, we introduce the Power-of-Two Leaky Integrate-and-Fire (PoT-LIF) neuron, which emits signed spikes with power-of-two amplitudes. By combining magnitude-coded spike events with residual-preserving membrane dynamics, PoT-LIF enables continuous information to be progressively represented through discrete spikes. The Heartbeat formulation further localizes this process within weighted projections, bridging biologically inspired spiking computation and pretrained VLA architectures without modifying their external inference interfaces. We integrate PoT-LIF into the vision projection, vision--language projection, and action projection of $\pi_{0.5}$ and evaluate both operating modes over 2,000 paired closed-loop LIBERO episodes per configuration. Standard-mode PoT-LIF VLA achieves success rates of 96.90\% and 96.35\% at $T=1$ and $T=2$, respectively, while Heartbeat mode achieves 97.25\% and 96.90\% at $T=2$ and $T=4$, compared with 96.70\% for the original continuous-activation $\pi_{0.5}$ baseline. At $T=1$ with a 60.01\% firing rate, PoT-LIF uses approximately $1/7.67$ of the dense-reference arithmetic energy, corresponding to a modeled reduction of 86.95\%.

\end{abstract}


\section{Introduction}\label{sec:1}

Vision--language--action (VLA) models couple pretrained vision--language backbones with action generation modules to enable closed-loop robotic manipulation in open environments~\citep{black2024pi0}. Systems such as $\pi_{0.5}$ \citep{physicalintelligence2025pi05}, Robotics Transformer 2 (RT-2) \citep{brohan2023rt2}, and OpenVLA \citep{kim2024openvla} have demonstrated strong cross-task generalization. However, VLA inference repeatedly executes dense weighted projections across vision, language, and action modules, resulting in high energy consumption on edge platforms. This motivates an alternative computational interface that preserves representational capacity while reducing arithmetic cost.

Spiking neural networks (SNNs) offer an event-driven computational paradigm~\citep{gerstner2014}. Leaky integrate-and-fire (LIF) neurons are widely adopted in SNNs. They convert continuous membrane potentials into sparse binary spikes, allowing synaptic computation to occur only when spikes are emitted.~\citep{lapicque1907,brunel2007lapicque}. However, a single binary spike does not explicitly encode the magnitude of the underlying membrane potential. To represent fine-grained continuous values, binary SNNs often accumulate information over multiple time steps, increasing computation and scheduling overhead. Recent studies, such as Spike-driven VLA and DS-VLA, have demonstrated the potential of SNNs for embodied control \citep{wang2026spikedriven,lyu2026dsvla}. Their integration strategies include directly trained spiking architectures and dendritic spiking action modules. However, these strategies can spread temporal processing across multiple modules, complicating integration into pretrained VLA architectures. These limitations motivate two key design goals for integrating spiking computation into pretrained VLA models: (1) representing continuous values with a limited number of events, and (2) confining temporal processing to selected projections without altering token layout, residual connections, or key-value (KV) cache organization.

To address the representation requirement, we draw inspiration from biological observations of exponential computation and graded signaling. Gabbiani et al. \citep{gabbiani2002} reported a nearly exponential mapping from membrane potential to firing rate in a locust visual neuron, consistent with multiplication through logarithmic input encoding and exponentiation by active membrane conductances. In cortical circuits, presynaptic membrane potential can also modulate axonal action potentials and the synaptic responses they evoke \citep{shu2006}. Motivated by these mechanisms, we use an exponent-based event representation: the amplitude takes the form ${b}^{e}$ where $b>1$ and $e$ is an integer. Successive exponent values specify multiplicative amplitude changes rather than equally spaced integer levels. We choose $b=2$ to align this representation with binary digital arithmetic. Residual retention then allows successive events to refine the represented value under a bounded emission budget.

We therefore propose a power-of-two leaky integrate-and-fire (PoT-LIF) neuron. A signed event has the form $s_t=\operatorname{sgn}(u_t)2^{e_t}$, where $\operatorname{sgn}(u_t)\in\{-1,+1\}$ and $e_t$ is an integer. For a floating-point weight $w$, an event contribution is evaluated through scaling--accumulation (SAC), formally defined in Sec.~\ref{sec:3.1}. For normal values whose scaled results remain normal and representable, scaling adjusts the weight exponent and combines the signs without a general significand multiplication. A residual membrane state retains the portion not represented by the current event, making the event budget $T$ an explicit control over precision and computational work.

The same PoT-LIF recurrence supports two operating modes: standard mode and Heartbeat. In standard mode, the neuron  receives successive inputs and retains its membrane state across input calls. Heartbeat instead applies the recurrence locally within each VLA projection: each projection invocation injects its input once and accumulates the weighted contributions of a bounded event sequence. The residual is retained across internal extraction steps, and the local state is reset at the end of the invocation. This confines the event-time dimension to individual projections while preserving the original token count, residual connections, and KV caches. Standard mode remains available for processing successive inputs and may support future integration with event-based sensors.

We evaluate PoT-LIF in the vision projection, vision--language projection, and action projection of $\pi_{0.5}$ using a fixed paired closed-loop protocol on LIBERO \citep{liu2023libero}.  Across both the main and supplementary evaluations, all tested PoT-LIF configurations of $\pi_{0.5}$ differ from the dense $\pi_{0.5}$ baseline by no more than 0.80 percentage points in success rate. We also estimate the arithmetic energy of PoT-LIF projection computation under a matched FP32 accumulation model, while excluding system-level overheads.

\newpage
This paper makes three contributions:
\nopagebreak[4]\begin{itemize}

\item \textbf{A unified PoT-LIF formulation with reconstruction guarantees.} We define a unified PoT-LIF recurrence with signed power-of-two emissions and residual retention, describe its standard mode, and formulate Heartbeat as a specific input schedule. We establish residual reconstruction guarantees under leak-free dynamics and exact arithmetic, with the geometric error bound requiring inactive exponent clipping (Sec.~\ref{sec:3}).

\item \textbf{PoT-LIF integration into $\pi_{0.5}$ with preserved closed-loop performance.} We integrate PoT-LIF into the vision projection, vision--language projection, and action projection of $\pi_{0.5}$ and evaluate standard mode followed by Heartbeat mode. Across both the main and supplementary evaluations, all tested PoT-LIF configurations differ from the 96.70\% dense $\pi_{0.5}$ baseline by no more than 0.80 percentage points in success rate (Sec.~\ref{sec:5}; paired statistical analyses are provided in Appendix~\ref{app:G}).

\item \textbf{Operation-level energy model showing a 7.67-fold arithmetic-energy advantage.} We formulate operation-level energy accounting for multiply--accumulate (MAC) and scaling--accumulation (SAC), the computation adapted to PoT-LIF and formally defined in Sec.~\ref{sec:3.1}. With matched FP32 accumulation and a 60.01\% firing rate at $T=1$, the modeled PoT-LIF arithmetic energy is approximately reach to 86.95\% reduction, excluding system-level overheads.

\end{itemize}

\section{Background and Related Work}\label{sec:2}

In this section, we first review the biological and computational foundations of LIF neurons, including membrane dynamics and surrogate-gradient training. We then survey related work on event representations, efficient arithmetic, and spiking VLA models.

\subsection{Spiking neural networks and LIF dynamics}\label{sec:2.1}

SNNs compute with discrete events rather than dense continuous-valued activations. A leaky integrate-and-fire (LIF) neuron integrates input into its membrane potential, exhibits passive leakage in the absence of input, emits a spike at threshold, and resets after emission (\citep{gerstner2014}, \cite{lapicque1907}, \cite{brunel2007lapicque}). The LIF dynamics can be written as
\nopagebreak[4]
\begin{equation}
\tau_m\frac{\mathrm{d}u}{\mathrm{d}t}=-u(t)+R_m I_{\mathrm{syn}}(t).
\label{eq:lif}
\end{equation}
Where $u(t)$ is the membrane potential, $\tau_m$ is the membrane time constant, $R_m$ is the membrane resistance, and $I_{\mathrm{syn}}(t)$ is the synaptic current.

At discrete step $t$, let $\theta>0$ be the firing threshold. The binary spike occurs when $u_{t} > \theta$, we denote the spike by adapting the Heaviside function
\nopagebreak[4]
\begin{equation}
	z(t) = H(u_t-\theta)=
	\begin{cases}
		1, & u_t\ge\theta,\\
		0, & u_t<\theta.
	\end{cases}
	\label{eq:binary}
\end{equation}

Assume that $r_t$ denotes the membrane state retained after emission and
reset. With zero resting potential, the hard and soft reset
rules are given, respectively, by
\nopagebreak[4]
\begin{equation}
	r_t=(1-z_t)u_t \quad \text{(hard reset)}, \qquad
	r_t=u_t-\theta z_t \quad \text{(soft reset)}.
\end{equation}

An exponentially decaying synaptic current $I_{\mathrm{syn}}(t)$
and the passive decay of the membrane potential in the absence
of input are described by the following expressions,
respectively \citep{destexhe1994,tsodyks1997}:
\nopagebreak[4]
\begin{equation}
	I_{\mathrm{syn}}(t)=A\exp(-t/\tau_s),\quad t\ge0,\qquad
	u(t+\Delta t)=\exp(-\Delta t/\tau_m)u(t),
	\label{eq:decay}
\end{equation}
here, $A$ is the synaptic current amplitude and $\tau_s$ and $\tau_m$ are the synaptic and membrane time constants. Discretizing this decay gives $\lambda=\exp(-\Delta t/\tau_m)$, where $\Delta t$ is the discrete time step.

A binary spike indicates whether a neuron fires; it does not explicitly encode the membrane-potential magnitude. Consequently, SNNs represent continuous quantities through rate coding, in which firing frequency carries the value, or temporal coding, in which spike timing carries the value \citep{gerstner2014,cao2015spiking,sengupta2019going}. Such representations generally require multiple time steps.

\subsection{Surrogate-gradient training}\label{sec:2.2}

Binary spike generation is modeled by the Heaviside step function, whose derivative is zero almost everywhere and undefined at the threshold. Consequently, direct backpropagation provides no useful gradients for learning. Surrogate-gradient methods enable gradient-based training through discontinuous spike-generation functions \citep{neftci2019surrogate}. The forward pass retains the binary spike, while the backward pass uses the derivative of a smooth substitute in place of the step-function derivative.

In this paper, PoT-LIF changes the event amplitude and residual retention; its discontinuous event-generation rule can still be trained using the same surrogate-gradient principle. Specifically, the forward pass retains the actual signed power-of-two events, while the backward pass approximates the derivatives of the discontinuous operations. The training configuration used in this work is described in Sec.~\ref{sec:5}.

\subsection{Event representations and efficient arithmetic}\label{sec:2.3}

Biological neurons exhibit a nearly exponential relationship between membrane potential and firing rate. Presynaptic membrane potential also modulates axonal action potentials and postsynaptic response amplitudes. These observations motivate introducing exponentially spaced firing levels and amplitude-bearing events into LIF dynamics. Theoretical analyses of neural coding under metabolic constraints \citep{levy1996} and experiments on energy-efficient synaptic transmission \citep{harris2015} further motivate considering the information conveyed by each event alongside its computational cost.

Logarithmic representations and DeepShift show how powers of two replace general multiplication with operations suited to binary hardware \citep{miyashita2016log,elhoushi2021deepshift}; AdderNet provides a complementary example of reducing reliance on multiplication through addition-based distance operations \citep{chen2020addernet}. Combining this arithmetic motivation with the biological observations leads to PoT-LIF: exponentially spaced levels follow $b^e$, where $b>1$ and $e$ is an integer, and choosing $b=2$ aligns the event amplitudes with binary digital arithmetic. The resulting events carry a sign and an integer exponent, while the membrane retains the unrepresented residual.

Neuromorphic hardware also supports carrying amplitude information in events. Intel's Loihi 2 supports graded spikes with integer-valued payloads of up to 32 bits; its technology brief reports that the commonly used 8-bit payload entails negligible extra energy cost relative to binary spike processing \citep{intel2021loihi2}. This provides engineering evidence that amplitude-bearing communication can preserve the energy advantages of sparse events, although the cost of a particular encoding depends on its implementation. Together, biological exponential responses, power-of-two arithmetic, and hardware support for graded events motivate the PoT-LIF representation.

\subsection{Spiking VLA models}\label{sec:2.4}

Wang et al. \citep{wang2026spikedriven} introduce Spike-driven VLA, which directly trains spiking perception, fusion, and action modules. Its 0.15B-parameter model achieves only 92.4\% mean LIBERO success, leaving task success as a key limitation.

Lyu et al. \citep{lyu2026dsvla} introduce DS-VLA, which combines dendritic spiking dynamics with inhibitory gating for robust action control. Although it improves resistance to action perturbations, its nominal LIBERO success rate remains only 91.6\%.

We identify limited amplitude information in binary spikes as a key
representational bottleneck, which limits success rates in spiking VLA control. This motivates PoT-LIF: signed power-of-two events convey magnitude information at low arithmetic cost, while residual retention refines the representation over successive events. We aim to improve end-to-end task success without sacrificing the arithmetic efficiency of event-based computation.

\section{Power-of-Two Leaky Integrate-and-Fire Neurons}\label{sec:3}

In this section, we define the PoT-LIF dynamics and SAC arithmetic. We then introduce the Heartbeat operating mode. Finally, we prove that under leak-free dynamics, Heartbeat guarantees geometric error reduction, exact recovery at zero residual, and reconstruction error bounded by θ upon threshold stopping.

\subsection{PoT-LIF event encoding and SAC}\label{sec:3.1}

PoT-LIF extends the discrete LIF recurrence in Sec.~\ref{sec:2.1} by replacing the binary spike with a signed power-of-two event and retaining the residual.

At each internal time step $t$, the pre-emission membrane potential
$u_t$ is obtained from the input contribution $I_t$ and the retained
residual $r_{t-1}$:
\nopagebreak[4]
\begin{equation}
	u_t=I_t+\lambda r_{t-1},
\end{equation}
where $\lambda$ is the membrane decay factor defined in
Sec.~\ref{sec:2.1}. In our implementation, the threshold is chosen on the event ladder:
\begin{equation}
	\theta=2^h,
	\qquad
	h\in\mathbb{Z},
	\qquad
	e_{\min}\le h\le e_{\max}.
	\label{eq:threshold_ladder}
\end{equation}
The full ladder contains $e_{\max}-e_{\min}+1=33$ positive levels,
or 66 signed nonzero levels.

Each nonzero event in PoTLIF is encoded as

\nopagebreak[4]
\begin{equation}
	s_t=\begin{cases}\operatorname{sgn}(u_t)2^{\lfloor\log_2|u_t|\rfloor},&|u_t|\ge\theta\ \text{and}\ u_t\ne0,\\0,&|u_t|<\theta\ \text{or}\ u_t=0,\end{cases}
	\label{eq:emission}
\end{equation}

On emitting steps, the ideal exponent is denoted by $e_t^{\mathrm{ideal}}=\lfloor\log_2|u_t|\rfloor$. In the finite-range implementation, the emitted exponent is clipped to

\nopagebreak[4]
\begin{equation}
	e_t=\begin{cases}e_{\min},&e_t^{\mathrm{ideal}}<e_{\min},\\e_t^{\mathrm{ideal}},&e_{\min}\le e_t^{\mathrm{ideal}}\le e_{\max},\\e_{\max},&e_t^{\mathrm{ideal}}>e_{\max},\end{cases}
	\label{eq:clipping}
\end{equation}
in this way, $e_{t}\in \mathbb{Z}$, and the exponent need not be evaluated when $s_t=0$. Thus, exponents within the prescribed window remain unchanged,
while those outside it are saturated to the corresponding boundary.
The exponent is evaluated only when an event is emitted. Each nonzero event therefore has
a magnitude in the discrete set
\begin{equation}
	|s_t|\in
	\left\{2^e:e\in\mathbb{Z},\ e_{\min}\le e\le e_{\max}\right\}.
	\label{eq:event_ladder}
\end{equation}
The exponent is carried in a signed 6-bit field, restricted to
$[-16,16]$, while the event sign is determined by
$\operatorname{sgn}(u_t)$ and stored in a separate sign bit.
This window specifies the event format rather than the weight format.

After emission, the remaining membrane potential is retained as
\nopagebreak[4]
\begin{equation}\label{eq:state_balance}
	r_t=u_t-s_t.
\end{equation}
Hence, in exact arithmetic,
\nopagebreak[4]
\begin{equation}
	s_t+r_t=u_t=I_t+\lambda r_{t-1},
\end{equation}
the residual $r_t$ is carried forward to the next internal time step. In this way, the pre-emission membrane potential is decomposed into the
emitted event and the retained residual. The setting $\lambda=1$
is the leak-free special case used by our heartbeat configuration.

Let $a$ denote an accumulator, $w$ a floating-point weight, and $s_t$ an
event emitted by PoT-LIF. In exact arithmetic, given an event $s_t$, 
 the scaling--accumulation map (SAC) is defined as
\begin{equation}
	\operatorname{SAC}(a,w,s_t)
	:= a+w s_t
	=
	\begin{cases}
		a+\operatorname{sgn}(u_t)
		\operatorname{Scale}_{2}(w,e_t),
		& s_t\ne0,\\
		a,
		& s_t=0,
	\end{cases}
	\label{eq:sac_definition}
\end{equation}
where $\operatorname{Scale}_{2}(w,e_t):=2^{e_t}w$.
The corresponding accumulator update is
\begin{equation}
	a\leftarrow\operatorname{SAC}(a,w,s_t).
	\label{eq:sac_update}
\end{equation}
Thus, SAC accumulates weighted event contributions, while
the event amplitudes are determined by the PoT-LIF emission rule.

For a receiving neuron $j$, the accumulator $a_j$ is initialized
to zero at each internal time step $t$. Applying SAC to the
presynaptic events $s_{t,i}$ with connection weights $w_{ji}$
produces the input contribution
\begin{equation}
	I_{t,j}=a_j=\sum_i \operatorname{SAC}(a,w_{ji}, s_{t,i}) .
	\label{eq:sac_input}
\end{equation}
The accumulated input is not restricted to power-of-two values. The receiving neuron then follows the same recurrence:
\nopagebreak[4]
\begin{equation}
	u_{t,j}=I_{t,j}+\lambda r_{t-1,j},
	\qquad
	r_{t,j}=u_{t,j}-s_{t,j}.
\end{equation}

\subsection{Heartbeat input schedule}\label{sec:3.2}

Heartbeat is an input schedule that applies the same recurrence to a single input within a weighted projection. At the start of each projection invocation, the local state is initialized to ${r}_{0}=0$, and the input is injected once:

\nopagebreak[4]
\begin{equation}
r_0=0,\qquad I_t=\begin{cases}x,&t=1,\\0,&t>1.\end{cases}
\label{eq:heartbeat-input}
\end{equation}

The neuron then executes at most $T$ event-extraction steps. The resulting approximation is

\nopagebreak[4]
\begin{equation}
q_T (x)=\sum_{t=1}^{T}s_t,
\label{eq:sum}
\end{equation}

At each tick, the pre-emission membrane potential, event, and
residual are computed using equations \eqref{eq:emission} and \eqref{eq:state_balance}. A local accumulator,
initialized to zero, is updated by $q\leftarrow q+s_t$ whenever
$s_t\neq 0$. Execution stops when the residual becomes zero,
when the emission gate produces no event, or when the tick
budget $T$ is exhausted.

For a vector input, the recurrence is applied elementwise, With projection weight matrix $W$, the event contributions are accumulated before the next nonlinearity: 

\nopagebreak[4]
\begin{equation}
Wq_T=\sum_{t=1}^{T}Ws_t.
\label{eq:linearity}
\end{equation}

The projection bias, if present, is added once. The event-time dimension therefore remains internal to the projection. In our configuration, the decay rate is set to $\lambda =1$, so a residual below the threshold cannot produce a later event in the absence of further input and extraction can stop immediately.

Figure~\ref{fig:1} illustrates the two operating modes. In Fig.~\ref{fig:1}a, presynaptic events carrying signed power-of-two amplitudes arrive through weighted synapses. The neuron integrates these contributions, emits events according to Eq.~\eqref{eq:emission}, and retains the residual for subsequent inputs. In Fig.~\ref{fig:1}b, Heartbeat injects a scalar input once and extracts its event representation over at most $T$ internal steps. With $\lambda =1$, the residual is retained without leakage between extraction steps and is reset at the end of the projection invocation.

\begin{figure}[H]\centering
\includegraphics[width=1.0\linewidth]{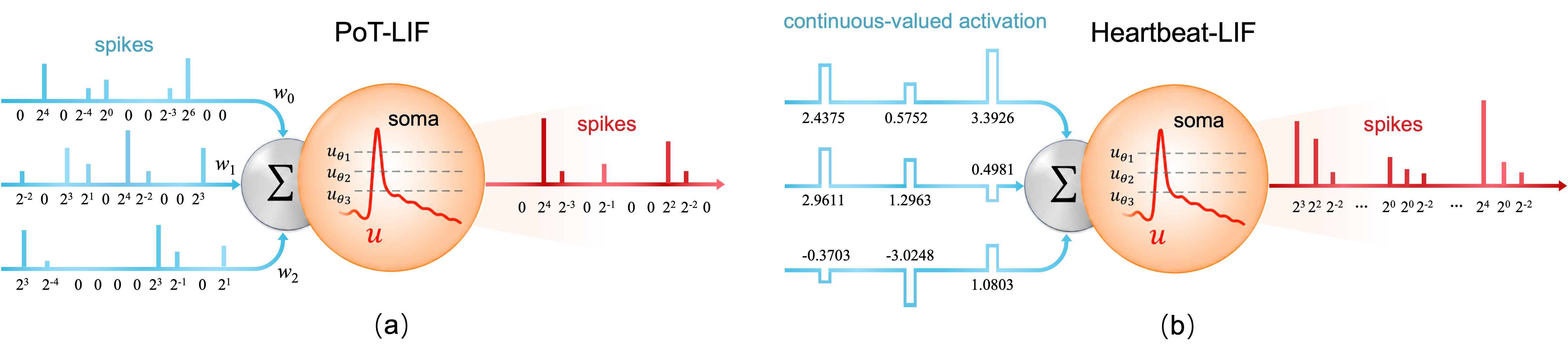}
\caption{PoT-LIF dynamics and power-of-two event encoding. (a) Standard mode: presynaptic events arrive through weighted synapses and are integrated into the membrane state. Each emitted event carries a sign and an integer exponent specifying its amplitude. The residual remains in the membrane state across successive input calls. (b) Heartbeat mode: a scalar input is injected once and decomposed into a bounded event sequence using the same recurrence with $\lambda =1$. The residual is retained between internal extraction steps, and the local state is reset at the end of the projection invocation.}\label{fig:1}
\end{figure}

\subsection{Reconstruction analysis}\label{sec:3.3}

In the reconstruction analysis, $r_T$ denotes the residual before the local state is reset. If extraction terminates at step $K<T$ under $\lambda=1$, we set $s_t=0$ and $r_t=r_K$ for the remaining analytical steps $t=K+1,\ldots,T$.

\begin{proposition}[Residual reconstruction for $\lambda=1$]\label{prop:residual}
Let $x\in\mathbb{R}$, the
initialization and input schedule is given in
Eq.~\eqref{eq:heartbeat-input}. In exact arithmetic, assume that $\lambda =1$, the exponent window $[e_{\min},e_{\max}]$, and the threshold
$\theta=2^h$, where $h\in\mathbb{Z}$ and
$e_{\min}\le h\le e_{\max}$, the unclipped emission rule in \eqref{eq:emission}. Let $K\le T$ denote the number of executed ticks, with
$r_K$ evaluated before the state reset. Then the input $x$ can be reconstructed: 
\begin{equation}
	x=q_T(x)+r_K.
	\label{eq:heartbeat_balance}
\end{equation}
In particular:  

(i) If an event is emitted at each of the $T$ steps and the residual remains nonzero, then

\nopagebreak[4]
\begin{equation}
|x-q_T(x)|\le2^{-T}|x|.
\label{eq:bound}
\end{equation}

(ii) If the residual reaches zero at some step ${t}_{0}\le T$, then

\nopagebreak[4]
\begin{equation}
q_T(x)=x.
\label{eq:exact}
\end{equation}

(iii) If extraction terminates since the remaining residual is below the threshold, then

\nopagebreak[4]
\begin{equation}
|x-q_T(x)|<\theta.
\label{eq:threshold-bound}
\end{equation}

\end{proposition}

In case (i), the input range condition and $\theta\ge 2^{e_{\min}}$ prevents upper and lower exponent
saturation on emitting ticks, respectively. Each nonzero residual is therefore
reduced by at least a factor of two. Upper exponent saturation can
invalidate this geometric contraction, but it does not invalidate
Eq.~\eqref{eq:heartbeat_balance} or cases (ii) and (iii).
The proof is given in Appendix~\ref{app:A}.

For $\lambda<1$, any nonzero retained state is attenuated before subsequent emissions. If $K\le T$ ticks are executed, summing the recurrence gives
\begin{equation}
x-q_T(x)=r_K+(1-\lambda)\sum_{t=1}^{K-1}r_t.
\label{eq:leak-balance}
\end{equation}
The reconstruction error therefore includes both the terminal residual and accumulated leakage. Consequently, $q_T(x)$ generally does not reconstruct the full input, and the reconstruction bound above does not apply in this form.

Our Heartbeat configuration fixes $\lambda =1$ and set the minimum supported event magnitude as $\theta ={2}^{{e}_{\operatorname{min}}}={2}^{-16}$. This threshold prevents emission when the residual is smaller than the encoding can represent. If a lower threshold were activated, clipping the exponent to ${e}_{\operatorname{min}}$ could produce an event larger than the residual and reverse its sign.

When upper clipping is inactive, the decomposition successively extracts the nonzero power-of-two components of the input down to the minimum supported amplitude (Sec.~\ref{sec:3.1}). Exact reconstruction is obtained when all nonzero components lie within the supported exponent range and $T$ is sufficient to extract them. Otherwise, extraction either reaches the budget or stops with a residual smaller than $\theta $. For example, $x=1.4375$ produces the event sequence $(1,\,0.25,\,0.125,\,0.0625)$. Two events reconstruct $1.25$, whereas four reconstruct the input exactly.

The worst-case bound permits a relative error of up to 25\% at $T=2$. For positive inputs, this limit is approached as $x$ approaches ${2}^{e+1}$ from below, provided clipping and the threshold do not interfere with extraction. For example, $x=1.99$ leaves a residual of $0.49$, corresponding to a relative error of approximately 24.6\%. The rule truncates rather than rounds: the residual contains the binary components below the second extracted nonzero bit. It vanishes when the input has at most two nonzero binary components and both lie within the supported range.

For the significands which distribut over $[1,2)$ uniformly, the mean relative error is approximately 9.9\% at $T=2$, with a 95th percentile of 23.1\%. At $T=4$, these values decrease to 1.1\% and 4.0\%, respectively. These statistics describe the specified significand distribution rather than the activation distribution of the VLA model.

\section{PoT-LIF Integration into Pretrained Vision--Language--Action Models}\label{sec:4}

We apply the recurrence defined in Sec.~\ref{sec:3} to the vision encoder, vision--language backbone, and flow-matching action expert of pretrained VLA models. In this work, we choose $\pi_{0.5}$ \citep{black2024pi0,physicalintelligence2025pi05} as the foundation model. We first evaluate standard mode, in which the event-time dimension extends across layers within a bounded network-level sequence, and then Heartbeat, which confines event generation and weighted accumulation to individual projections. Each selected projection retains its original input and output shapes, preserving token organization, residual connections, KV-cache structure, and the action-generation interface.

\subsection{Projection-level event dataflow}\label{sec:4.1}

Figure~\ref{fig:2} shows the projection sites in the three principal components of $\pi_{0.5}$. Figure~\ref{fig:2}(a) follows the visual inputs through the Sigmoid Loss for Language--Image Pre-training (SigLIP) vision encoder \citep{zhai2023siglip} and PaliGemma vision--language model (VLM) \citep{beyer2024paligemma} to the action expert. Figure~\ref{fig:2}(b) illustrates how a PoT-LIF encoder supplies events to each selected weighted projection.

We incorporate PoT-LIF into 26 specific projection sites across six functional classes: patch embedding, query/key/value (Q/K/V), attention output, gated feed-forward layers, expert projections, and expert interfaces.

The standard-mode configurations use $T=1$ or $T=2$ internal time steps across the network for each image-and-prompt input. Membrane states are reset when that input is supplied and retained between internal steps as events propagate across layers. The $T$ output states are averaged at the velocity-field readout. State does not carry into the next image-and-prompt input.

In Heartbeat mode, Q/K/V, attention-output, and feed-forward projections accumulate the weighted contributions of a bounded event sequence. Attention-score computation, softmax, normalization, and residual addition retain their original operations. Event contributions are combined before the next nonlinearity, so the internal extraction steps do not expand the token sequence or replicate the KV cache. The projection output has the same shape as in the continuous-activation policy.

For consistency with Sec.~\ref{sec:3}, we write a projection using column-vector notation, $y=Wx$, with bias omitted. Heartbeat produces the approximation $\widehat{y}=W{q}_{T}(x)$ by accumulating at most $T$ events per input element:

\nopagebreak[4]
\begin{equation}
	\widehat{y}=\sum_{t=1}^{T}Ws_t,\qquad Ws_t=\sum_{i:s_{t,i}\ne0}\sigma_{t,i}\operatorname{Scale}_{2}(W_{:,i},e_{t,i}).
	\label{eq:projection}
\end{equation}

Here, ${W}_{:,i}$ is the weight column associated with input element $i$, and $\operatorname{Scale}_{2}$ applies power-of-two scaling to its floating-point entries. The inner sum includes only nonzero events. For normal, representable results, scaling adjusts the weight exponent and combines the event and weight signs without a general significand multiplication.

Under exact arithmetic, the projection error is $\widehat{y}-y=W({q}_{T}(x)-x)$. Thus, the residual bounds in Sec.~\ref{sec:3} characterize the input reconstruction error; its effect on the projection also depends on $W$. Finite-precision scaling and accumulation can introduce additional numerical error. SAC requires event-exponent decoding, exponent addition, and sign handling before accumulation.  

At $T=1$, standard mode and Heartbeat are equivalent under the evaluated design: each projection performs a single emission step per input element from a zero initial state, with no subsequent temporal update. We report this shared configuration as standard mode at $T=1$. At $T=2$, standard mode retains a network-wide temporal sequence, whereas Heartbeat performs both extraction steps locally within each projection.

This scheduling distinction also affects weight movement. On conventional hardware that reloads projection weights for each network pass, standard mode at $T=2$ requires two rounds of weight movement, in addition to repeated network execution and membrane-state access. Replacing multiplication with SAC therefore does not by itself establish an overall energy advantage for this implementation. This limitation motivates Heartbeat: processing consecutive events within a single projection allows a suitable datapath to reuse resident weight tiles across extraction steps before advancing to the next layer. 

\begin{figure}[H]\centering
\includegraphics[width=0.9\linewidth]{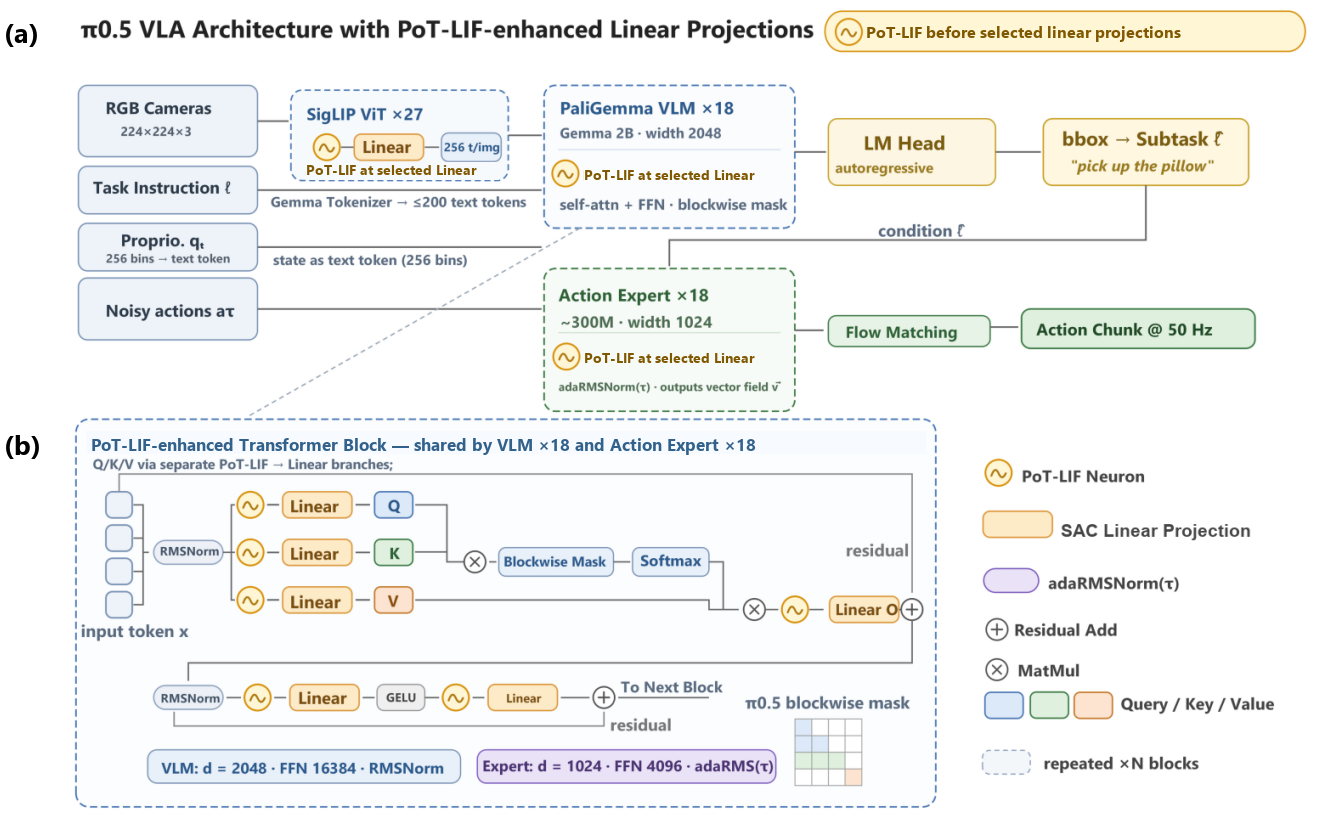}
\caption{(a) PoT-LIF integration into $\pi_{0.5}$. Heartbeat encoders supply signed power-of-two events to selected weighted projections in the vision encoder, vision--language backbone, and action expert. Event contributions are evaluated through scaling--accumulation (SAC) within each projection. Attention computation, the blockwise mask, KV-cache structure, residual paths, and action interfaces retain their original organization. Dashed boxes indicate the three module groups; (b) shows the event--projection pairs within a block. Abbreviations: RGB, red--green--blue; ViT, vision transformer; LM, language model; FFN, feed-forward network; RMSNorm, root mean square normalization; adaRMSNorm, adaptive RMSNorm; GELU, Gaussian error linear unit; bbox, bounding box; proprio., proprioceptive state; MatMul, matrix multiplication.}\label{fig:2}
\end{figure}

\section{Robot Experiments}\label{sec:5}

We initialize the PoT-LIF policies from pretrained $\pi_{0.5}$ and fine-tune them for 5,000 optimizer updates with an effective batch size of 256. The main configurations use the symmetric emission threshold $\theta=2^{-16}$ and $\lambda=1$. The forward pass uses signed power-of-two events, while the backward pass uses a straight-through estimator (STE), which substitutes a prescribed gradient for the discontinuous event-generation rule \citep{neftci2019surrogate}. The matched-budget continuous control receives the same fine-tuning budget. Detailed training settings are provided in Appendix~\ref{app:E}.

\subsection{Protocol}\label{sec:5.1}

We evaluate on the Spatial, Object, Goal, and LIBERO-10 (Long-Horizon) suites \citep{liu2023libero}, with ten tasks and 500 closed-loop episodes per suite, giving 2,000 episodes per configuration. Conditions share initial states, paired inference noise, and the action interface (evaluation seed 7). Two RGB views, from the frontal base and left wrist cameras, are resized to $224\times 224$; a third wrist slot is inactive. The policy predicts 10-step action chunks and replans every five environment steps. Actions between replanning points are executed open loop. PoT-LIF internal steps, flow-solver steps, and environment steps are distinct time axes.

\subsection{Main results}\label{sec:5.2}

Standard mode achieves 96.90\% overall success at $T=1$ and 96.35\% at $T=2$, compared with 96.70\% for the continuous-activation control. Heartbeat achieves 97.25\% at $T=2$ and 96.90\% at $T=4$ (Table~\ref{tab:1}). Success rates remain essentially unchanged under the evaluated protocol; paired significance tests and detailed results are reported in Appendix~\ref{app:G}.

{\small
\setlength{\tabcolsep}{4pt}
\renewcommand{\arraystretch}{1.2}
\begin{longtable}{@{}>{\raggedright\arraybackslash}p{\dimexpr 0.34\linewidth-1.7142857142857142\tabcolsep\relax}>{\raggedright\arraybackslash}p{\dimexpr 0.11\linewidth-1.7142857142857142\tabcolsep\relax}>{\raggedright\arraybackslash}p{\dimexpr 0.11\linewidth-1.7142857142857142\tabcolsep\relax}>{\raggedright\arraybackslash}p{\dimexpr 0.1\linewidth-1.7142857142857142\tabcolsep\relax}>{\raggedright\arraybackslash}p{\dimexpr 0.12\linewidth-1.7142857142857142\tabcolsep\relax}>{\raggedright\arraybackslash}p{\dimexpr 0.12\linewidth-1.7142857142857142\tabcolsep\relax}>{\raggedright\arraybackslash}p{\dimexpr 0.1\linewidth-1.7142857142857142\tabcolsep\relax}@{}}
\caption{Main LIBERO results (success rate per suite over 500 episodes; 2,000 per configuration). $\Delta$ is the overall percentage-point difference from the continuous-activation control.}\label{tab:1}\\
\toprule
Configuration & Spatial & Object & Goal & Long-H. & Overall & $\Delta $ \\
\midrule\endfirsthead
\multicolumn{7}{l}{\small\itshape Table \thetable\ (continued)}\\\toprule
Configuration & Spatial & Object & Goal & Long-H. & Overall & $\Delta $ \\
\midrule\endhead
\midrule\multicolumn{7}{r}{\small\itshape Continued on next page}\\\endfoot\bottomrule\endlastfoot
Continuous activation  & 98.8 & 98.6 & 97.0 & 92.4 & 96.70 & -- \\
Continuous, matched budget & 98.0 & 97.6 & 97.8 & 91.6 & 96.25 & $-$0.45 \\
Standard mode $T=1$ & 99.2 & 98.4 & 97.6 & 92.4 & 96.90 & $+$0.20 \\
Standard mode $T=2$ & 98.2 & 97.8 & 97.4 & 92.0 & 96.35 & $-$0.35 \\
Heartbeat $T=2$  & 99.2 & 97.4 & 98.2 & 94.2 & 97.25 & $+$0.55 \\
Heartbeat $T=4$  & 99.2 & 98.0 & 97.8 & 92.6 & 96.90 & $+$0.20 \\
\end{longtable}}

Performance varies across suites, with no single configuration achieving the highest success rate on all four. The four event-based configurations achieve 97.4--98.4\% success on Object, compared with 98.6\% for the continuous-activation control. On Long-Horizon, Heartbeat at $T=2$ achieves 94.2\%, compared with 92.4\% for the control. Additional ablations of the architecture and training procedure are reported in Appendix~\ref{app:C}.

\subsection{Event statistics and the choice of $T$}\label{sec:5.3}

Let ${n}_{l}$ be the number of input elements processed at projection site $l$, ${a}_{l}$ the number of nonzero events emitted there, and ${d}_{l}$ its output width. The product-weighted event ratio is $\mu =(\sum _{l}^{}{d}_{l}{a}_{l})/(\sum _{l}^{}{d}_{l}{n}_{l})$. Weighting by ${d}_{l}$ counts the event--weight pairs across all outputs, including fused QKV projections. For the sampled workload, ${\mu }_{2}=1.900$ and ${\mu }_{4}=3.480$. The projection-activity accounting covers 26 instrumented sites under the eight-observation sampling protocol described below. Its dense reference count is $N=\sum_l d_l n_l=2.271\times {10}^{12}$ MACs, counting one weight--input product and its accumulation for each input--output connection. This count covers the instrumented weighted projections, not all operations in a complete end-to-end policy inference. Heartbeat with $T=2$ emits $1.240\times {10}^{9}$ events, corresponding to $4.314\times {10}^{12}$ event--weight pairs; $T=4$ emits $2.270\times {10}^{9}$ events, corresponding to $7.902\times {10}^{12}$ pairs. Each nonzero event--weight pair contributes one SAC operation.

We use $T=2$ as the default Heartbeat setting because it combines 97.25\% success with fewer event--weight operations than $T=4$. The latter requires approximately 1.83 times as many SAC operations, while scoring 96.90\%. Operation counts use eight fixed observations (one environment observation and seven demonstration frames) over the 26 projection sites.

\section{Operation-Level Energy Accounting}\label{sec:6}

A reference product is one dense weight--input product; its event-coded replacement may require multiple event--weight products. We compare the arithmetic energy of dense FP32 MAC computation with PoT-LIF SAC computation using matched FP32 accumulation. The comparison uses an illustrative 45-nm arithmetic model and includes only multiplication or exponent scaling and accumulation. Data movement, memory access, control logic, event encoding, and other system-level costs are excluded.

The dense reference uses 4.6~pJ per FP32 MAC~\citep{sengupta2019going,horowitz2014energy}. For SAC, we use 0.9~pJ for FP32 addition and a 0.1~pJ integer-addition proxy for exponent addition, giving an illustrative arithmetic cost of 1.0~pJ per event--weight pair. The exponent-addition term uses a 32-bit integer-addition reference as a proxy. The floating-point addition cost represents the accumulation arithmetic, including its associated alignment, normalization, and rounding.

The 2.271 trillion reference MACs come from the 26-site projection-activity accounting in Sec.~\ref{sec:5.3}, based on eight fixed observations (one environment observation and seven demonstration frames). They count the covered weight--input products, excluding attention-score products, softmax, normalization, and other operations outside those projections. Multiplying this projection workload by 4.6~pJ per MAC gives an estimated dense arithmetic energy of 10.45~J. For the new $T=1$ setting, using the reported 60.01\% firing rate as the effective event--weight activity fraction gives $\mu_1=0.6001$. Applied to the same reference workload, this implies approximately 1.363 trillion SAC operations and 1.36~J, corresponding to an 86.95\% reduction. The sampled $T=2$ count of 4.314 trillion SAC operations gives 4.31~J, corresponding to a 58.7\% reduction. For $T=4$, 7.902 trillion SAC operations give 7.90~J, corresponding to a 24.4\% reduction.

The arithmetic-energy advantage depends on both the SAC unit cost and the number of events. With $\mu_1=0.6001$ at $T=1$ and an assumed SAC arithmetic cost of 1.0~pJ, the reference-to-SAC energy ratio is $4.6/(0.6001\times1.0)\approx7.67$, corresponding to an 86.95\% energy saving. Increasing the assumed SAC cost to 1.1 or 1.2~pJ gives 0.66011 or 0.72012~pJ per reference product, reducing the estimated energy saving to 85.65\% or 84.35\%, respectively.

For $f$ reference products per second, modeled arithmetic power is $4.6f$~pW for FP32 MAC and $0.6001f$~pW for PoT-LIF ($T=1$, firing rate 60.01\%).

\begin{figure}[!htbp]
	\centering
	\includegraphics[width=0.9\linewidth]{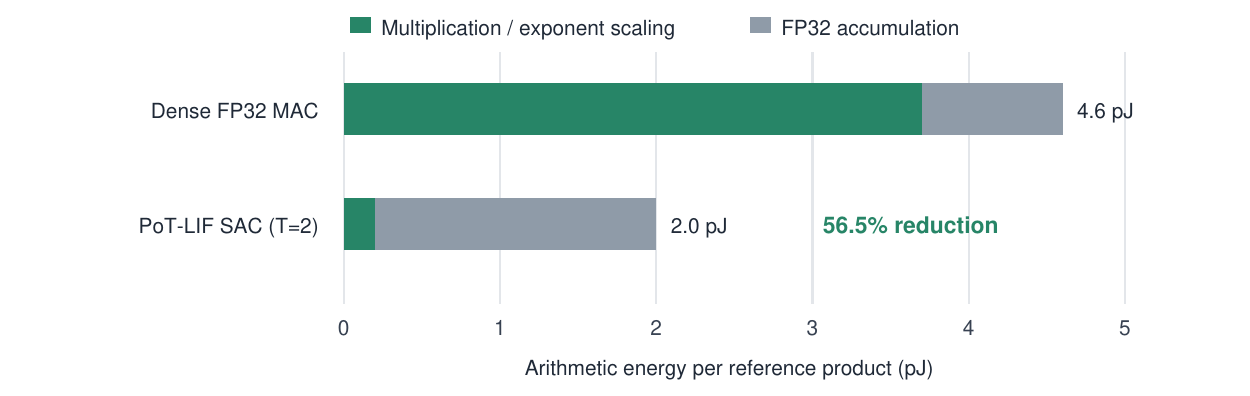}
	\caption{Average arithmetic energy per reference product with matched FP32 accumulation. One FP32 MAC costs 4.6~pJ. At $T=1$ and a 60.01\% firing rate, the PoT-LIF SAC path averages 0.6001~pJ: 0.06001~pJ for exponent scaling and 0.54009~pJ for accumulation. The reference-to-SAC energy ratio is 7.67, corresponding to a modeled reduction of 86.95\%. Other power terms are excluded.}
	\label{fig:4}
\end{figure}

We estimate the energy of PoT-LIF relative to a dense FP8 E4M3 baseline using matched FP8 weights and FP32 accumulation, with a 45\,nm reference cost model augmented by engineering proxies for exponent and control operations.
For $T=1$ and $T=2$ ticks with exponent-window half-width $k=4$, sampled fanout-weighted activity yields $0.6001$ and $1.0565$ SAC contributions per original dense projection product, respectively.
The modeled projection-arithmetic cost consequently decreases from $1.4646$\,pJ to $0.5926$\,pJ and $1.0432$\,pJ, corresponding to reductions of $59.54\%$ and $28.77\%$.
Including neuron encoding and the other modeled computation gives compute-only savings of $56.25\%$ and $27.16\%$.
Under an identical fixed data-movement budget of $2.518$\,J per policy inference (ten denoising steps), the modeled dynamic-energy subtotal decreases from $6.038$\,J to $4.058$\,J and $5.082$\,J, yielding savings of $32.79\%$ and $15.83\%$.

\begin{figure}[!htbp]
	\centering
	\includegraphics[width=0.84\linewidth]{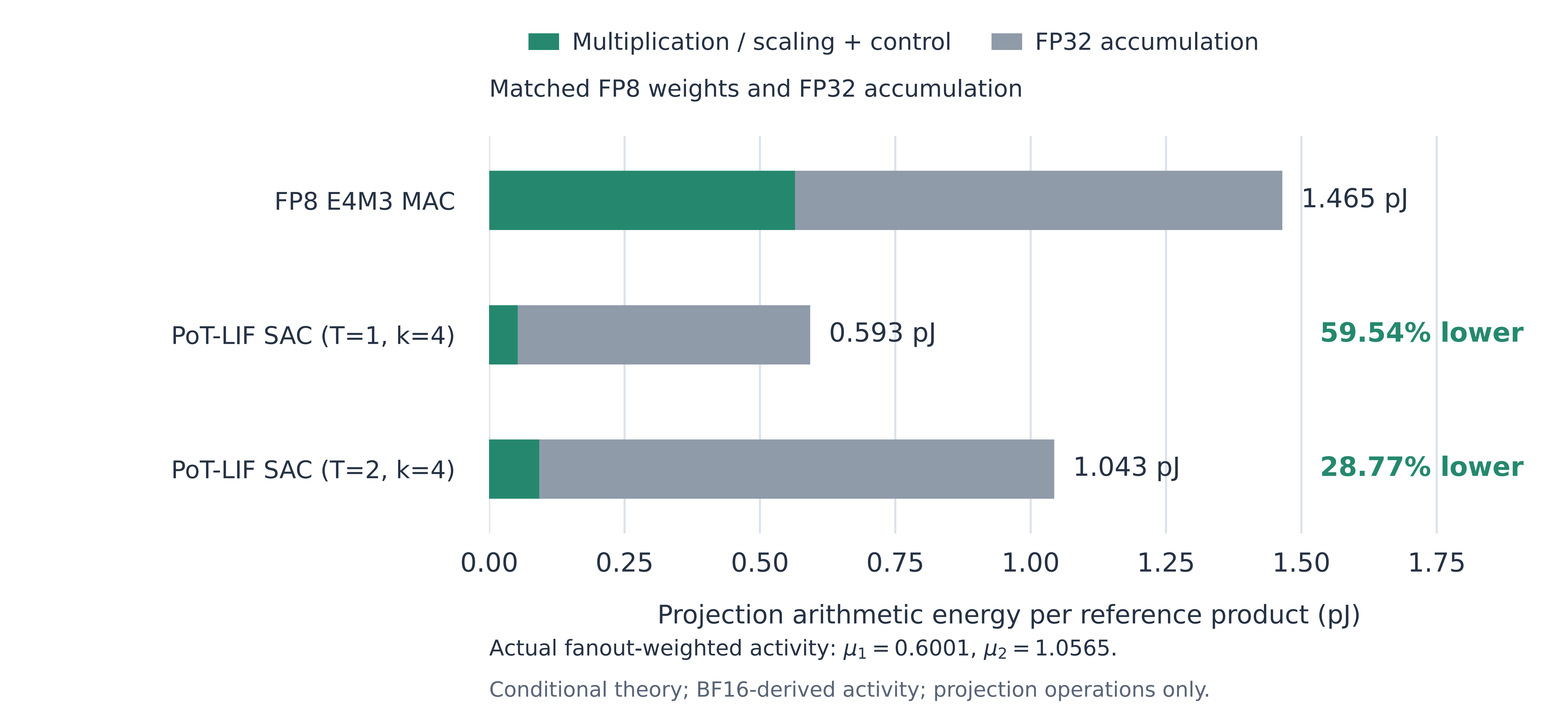}
	\caption{\textbf{Modeled projection arithmetic energy relative to an FP8 E4M3 baseline.}
		Both methods use FP8 weights and FP32 accumulation.
		Energy is normalized per original dense projection product using sampled fanout-weighted activity: $0.6001$ and $1.0565$ SAC contributions for $T=1$ and $T=2$, respectively, with $k=4$.
		Green denotes multiplication or exponent scaling and control; gray denotes FP32 accumulation.
		PoT-LIF reduces modeled projection arithmetic energy by $59.54\%$ and $28.77\%$.
		These estimates conditionally transfer BF16-derived activity to FP8 weights and exclude neuron encoding, shared computation, and data movement.}
	\label{fig:5}
\end{figure}

\begin{figure}[!htbp]
	\centering
	\includegraphics[width=0.84\linewidth]{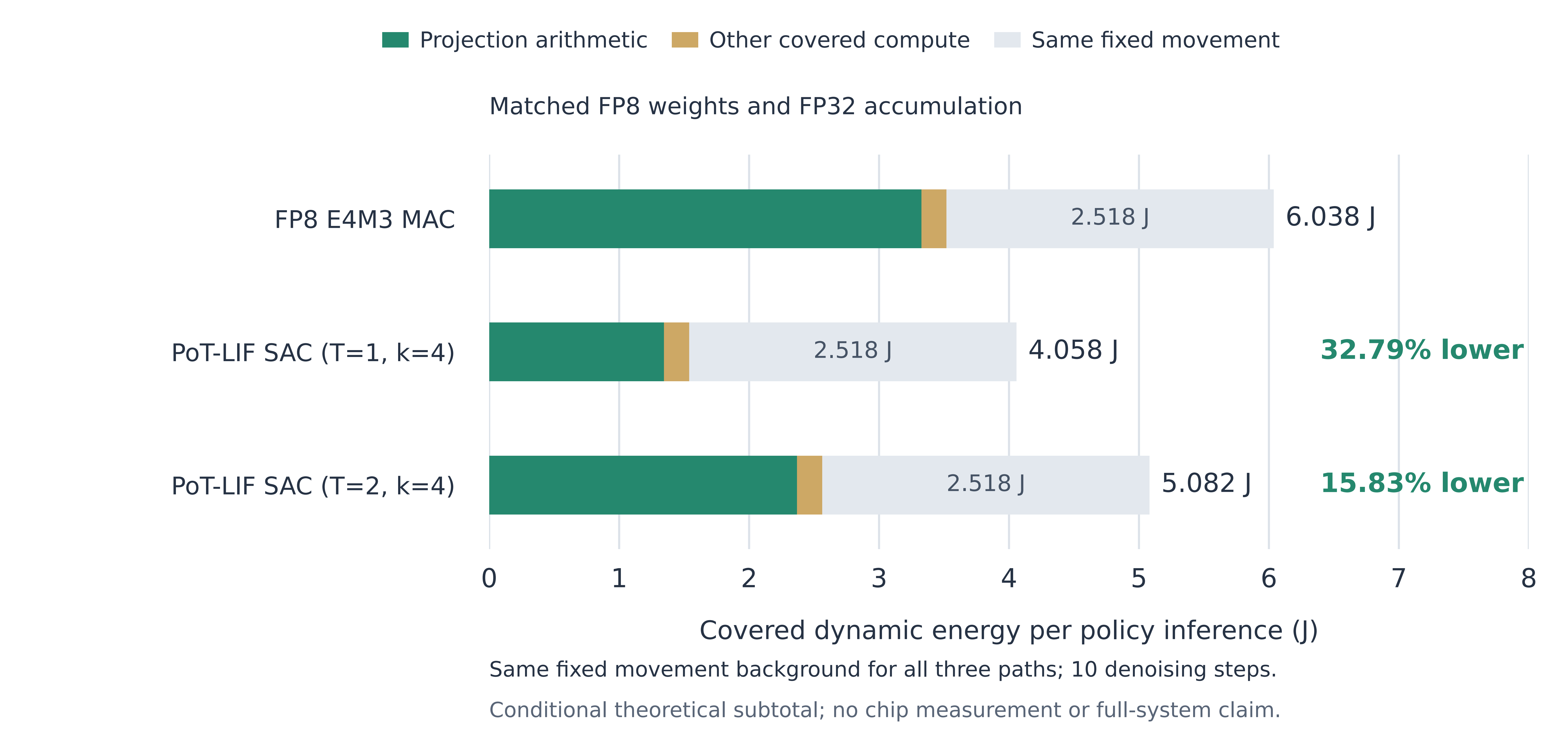}
	\caption{\textbf{Modeled dynamic energy per policy inference with a matched fixed data-movement budget.}
		Each inference includes ten denoising steps; all configurations use FP8 weights and FP32 accumulation.
		Green denotes projection arithmetic, gold denotes neuron encoding and other modeled computation, and light gray denotes the identical $2.518$\,J data-movement budget.
		With $k=4$, the modeled subtotal decreases from $6.038$\,J to $4.058$\,J for $T=1$ and $5.082$\,J for $T=2$, corresponding to savings of $32.79\%$ and $15.83\%$.
		Results are conditional on BF16-derived activity and represent partial energy estimates, not hardware measurements or complete system energy.}
	\label{fig:6}
\end{figure}

\FloatBarrier
\section{Discussion}\label{sec:7}

As demonstrated, PoT-LIF combines exponent-based event amplitudes with residual membrane dynamics. Standard mode and Heartbeat use the same recurrence but organize its execution differently: standard mode retains temporal processing across layers, while Heartbeat confines event generation and accumulation to individual projections. This local schedule preserves the pretrained policy's external token and cache organization and allows the event budget to be set within each projection.

The two schedules offer different paths toward deployment. For the current frame-based VLA workload, the repeated weight movement of standard mode at $T=2$ adds to its energy cost on conventional hardware. Reducing data-movement costs would allow a larger share of the arithmetic savings from power-of-two event coding to translate into end-to-end energy savings. Near-memory and compute-in-memory implementations that support resident weights, local membrane state, and efficient event accumulation offer a route to reducing the repeated weight-access cost of standard mode. Future work will co-design standard-mode event scheduling, membrane-state placement, and weight reuse for these architectures. Standard mode also retains the potential to process asynchronous event data because membrane state evolves with successive inputs.

The reconstruction guarantees in Sec.~\ref{sec:3.3} characterize the fidelity of the event representation, whereas closed-loop success also depends on how projection errors propagate through the policy and affect action generation. A larger event budget can therefore improve scalar reconstruction under the stated conditions but need not yield higher task success. In our experiments, increasing the Heartbeat budget from $T=2$ to $T=4$ increases the number of event--weight operations without improving the observed success rate. This supports selecting the event budget jointly from task performance and computational work, rather than from reconstruction accuracy alone.

Finally, regarding arithmetic efficiency, under the per-product accounting used here, the $T=1$ SAC path at a 60.01\% firing rate averages 0.6001~pJ per reference product. Its modeled arithmetic energy is approximately reach to an 86.95\% reduction.

\section{Conclusion}\label{sec:8}

This work introduces PoT-LIF and successfully integrated its signed power-of-two events and residual-preserving dynamics into pretrained $\pi_{0.5}$. The proposed Heartbeat mode achieves 97.25\% LIBERO success compared with 96.70\% for the original policy, preserving task performance while replacing dense projection multiplication with event-based scaling and accumulation. Under the matched operation-level energy model, the $T=1$ path at a 60.01\% firing rate uses approximately $1/7.67$ of the dense-reference arithmetic energy, corresponding to an 86.95\% reduction.

PoT-LIF is broadly applicable beyond $\pi_{0.5}$ to VLA models with compatible weighted-projection architectures. By localizing event processing within projections, Heartbeat preserves token layout, residual connections, KV-cache organization, and the action interface. This offers a route to easing the energy demands of edge VLA deployment while maintaining task success, as demonstrated here on $\pi_{0.5}$. In conclusion, PoT-LIF enables energy-efficient VLA inference by preserving closed-loop task performance while reducing arithmetic costs, paving the way for more efficient and sustainable embodied AI.

\section*{Acknowledgments}

We thank the open-source robot-learning and neuromorphic-computing communities for the software, benchmarks, and published research that enabled this study.

\clearpage
\appendix

\section*{Appendix}

\section{Proof of Proposition~\ref{prop:residual}}
\label{app:A}

\begin{proof}
Under the assumptions of Proposition~\ref{prop:residual}, $r_0=0$, $I_1=x$, $\lambda=1$, and $I_t=0$ for $t>1$. Thus $u_1=x$, while $u_t=r_{t-1}$ for $t>1$. Summing $r_t=I_t+r_{t-1}-s_t$ gives
\[r_T=x-q_T(x).\]
When an event fires under the unclipped rule,
\[|s_t|=2^{\lfloor\log_2|u_t|\rfloor},\qquad |s_t|\le |u_t|<2|s_t|.\]
Because the event has the same sign as the potential,
\[|r_t|=|u_t|-|s_t|<\tfrac12|u_t|.\]
(i) If all $T$ steps emit and the residual remains nonzero, repeated application gives $|x-q_T(x)|=|r_T|\le2^{-T}|x|$.
(ii) If $r_{t_0}=0$ for some $t_0\le T$, all subsequent events are zero and $q_T(x)=x$.
(iii) If extraction stops because $|u_t|<\theta$, the residual is held fixed for the remaining analytical steps, so $|x-q_T(x)|=|r_T|<\theta$.
For $\lambda<1$, summing over the $K$ executed ticks gives Eq.~\eqref{eq:leak-balance}; the accumulated leakage prevents the same residual-only reconstruction bound. The halving argument requires inactive exponent clipping, as specified in Sec.~\ref{sec:3.3}.
\end{proof}

\section{Configurable Response: Definition and Fits}\label{app:B}

Let $\{{\theta }_{k}\}$ be ordered thresholds and ${c}_{k}$ the increment associated with threshold $k$, represented by a bounded sum of signed power-of-two terms. The output response is

\nopagebreak[4]
\begin{equation}
f_{\Theta,\mathcal{C}}(x)=c_0+\sum_k c_k\mathbf{1}[x\ge\theta_k],\qquad c_k=\sum_{j=1}^{B_k}b_{kj}2^{e_{kj}}.
\label{eq:response}
\end{equation}

Here, ${c}_{0}$ is the response for $x<{\theta }_{1}$, ${b}_{kj}\in \{-1,+1\}$, ${e}_{kj}$ is an integer exponent, and ${B}_{k}$ is the local pulse budget. Threshold positions set the segment boundaries, and the pulse codes separately set the increments; $c_k$ is not constrained to equal $\theta_k$. These parameters can be fixed to construct a response or optimized using a straight-through estimator (STE), which replaces a discontinuous forward derivative with a prescribed backward value, \citep{neftci2019surrogate}. The VLA experiments use the greedy extraction rule of Sec.~\ref{sec:3.1} and retain the original network nonlinearities; they do not use this threshold table. The budget ${B}_{k}$ per threshold increment is distinct from the event budget $T$ per projection input. Figure~\ref{fig:3} in Sec.~\ref{sec:5.4} reports fits to three reference nonlinearities using this construction.

\section{Supplementary Surrogate and Training Comparisons}\label{app:C}

{\small
\setlength{\tabcolsep}{4pt}
\renewcommand{\arraystretch}{1.2}
\begin{longtable}{@{}>{\raggedright\arraybackslash}p{\dimexpr 0.47\linewidth-1.3333333333333333\tabcolsep\relax}>{\raggedright\arraybackslash}p{\dimexpr 0.32\linewidth-1.3333333333333333\tabcolsep\relax}>{\raggedright\arraybackslash}p{\dimexpr 0.21\linewidth-1.3333333333333333\tabcolsep\relax}@{}}
\caption{STE training ablation over 2,000 paired LIBERO episodes per configuration. Both runs start from the same pretrained checkpoint and are fine-tuned separately with the indicated surrogate.}\label{tab:2}\\
\toprule
Configuration & Surrogate setting & Success (\%) \\
\midrule\endfirsthead
\multicolumn{3}{l}{\small\itshape Table \thetable\ (continued)}\\\toprule
Configuration & Surrogate setting & Success (\%) \\
\midrule\endhead
\midrule\multicolumn{3}{r}{\small\itshape Continued on next page}\\\endfoot\bottomrule\endlastfoot
Heartbeat $T=4$, identity STE & identity surrogate & 96.90 \\
Heartbeat $T=4$, magnitude STE & magnitude surrogate & 96.90 \\
\end{longtable}}

The identity-STE and magnitude-STE Heartbeat models at $T=4$ are fine-tuned separately from the same initial $\pi_{0.5}$ checkpoint, yielding a distinct final checkpoint for each surrogate. Both achieve 96.90\% success under the paired evaluation protocol. The matched-budget continuous control (96.25\%) continues fine-tuning $\pi_{0.5}$ with continuous activations for the same number of updates as the event-based variants, providing a reference that accounts for the additional training budget. Standard mode at $T=2$ achieves 96.35\% (Table~\ref{tab:1}). Relative to Heartbeat, it changes both temporal execution across layers and the velocity-field readout, so this comparison assesses their combined effect.

\section{Event-Representation Comparison}\label{app:D}

Table~\ref{tab:3} compares our result with the reported nominal LIBERO success rates of other spiking VLA models.

{\small
	\setlength{\tabcolsep}{4pt}
	\renewcommand{\arraystretch}{1.2}
	\begin{longtable}{@{}>{\raggedright\arraybackslash}p{\dimexpr 0.25\linewidth-1.5\tabcolsep\relax}>{\raggedright\arraybackslash}p{\dimexpr 0.14\linewidth-1.5\tabcolsep\relax}>{\raggedright\arraybackslash}p{\dimexpr 0.41\linewidth-1.5\tabcolsep\relax}>{\raggedright\arraybackslash}p{\dimexpr 0.2\linewidth-1.5\tabcolsep\relax}@{}}
		\caption{Reported nominal LIBERO success of spiking VLA models. Results use each study's architecture and training protocol; --- indicates an unlisted parameter count. The PoT-LIF row repeats the Heartbeat $T=2$ result in Table~\ref{tab:1}.}\label{tab:3}\\
		\toprule
		Method & Parameters & Implementation & LIBERO mean \\
		\midrule\endfirsthead
		\multicolumn{4}{l}{\small\itshape Table \thetable\ (continued)}\\\toprule
		Method & Parameters & Implementation & LIBERO mean \\
		\midrule\endhead
		\midrule\multicolumn{4}{r}{\small\itshape Continued on next page}\\\endfoot\bottomrule\endlastfoot
		Spike-driven VLA \citep{wang2026spikedriven} & 0.15B & fully spiking, direct training & 92.4\% \\
		DS-VLA \citep{lyu2026dsvla} & --- & dendritic spiking action control & 91.6\% \\
		PoT-LIF on $\pi_{0.5}$ & 3.3B & Heartbeat $T=2$ projections & \textbf{97.25\%} \\
\end{longtable}}

PoT-LIF achieves the highest reported nominal success among the models in Table~\ref{tab:3}: 97.25\%, exceeding Spike-driven VLA by 4.85 percentage points and DS-VLA by 5.65 points. These results highlight the task-success advantage of combining amplitude-bearing events with a strong pretrained VLA.

Table~\ref{tab:4} compares the event representations and integration strategies of the spiking VLA systems discussed in Sec.~\ref{sec:2}. The two PoT-LIF modes share signed power-of-two events and residual membrane updates. Standard mode propagates events across layers, whereas Heartbeat accumulates their weighted contributions within each projection. The distinction concerns the execution schedule of the recurrence, as described in Secs.~\ref{sec:3}--\ref{sec:4}.

{\small
\setlength{\tabcolsep}{4pt}
\renewcommand{\arraystretch}{1.2}
\begin{longtable}{@{}>{\raggedright\arraybackslash}p{\dimexpr 0.23\linewidth-1.3333333333333333\tabcolsep\relax}>{\raggedright\arraybackslash}p{\dimexpr 0.34\linewidth-1.3333333333333333\tabcolsep\relax}>{\raggedright\arraybackslash}p{\dimexpr 0.43\linewidth-1.3333333333333333\tabcolsep\relax}@{}}
\caption{Event representations and integration strategies. Entries summarize the cited models; the PoT-LIF rows refer to the $\pi_{0.5}$ implementation evaluated here.}\label{tab:4}\\
\toprule
Model & Event representation & Integration strategy \\
\midrule\endfirsthead
\multicolumn{3}{l}{\small\itshape Table \thetable\ (continued)}\\\toprule
Model & Event representation & Integration strategy \\
\midrule\endhead
\midrule\multicolumn{3}{r}{\small\itshape Continued on next page}\\\endfoot\bottomrule\endlastfoot
Spike-driven VLA \citep{wang2026spikedriven} & Binary spikes & Spiking perception, fusion, and action modules \\
DS-VLA \citep{lyu2026dsvla} & Spiking dynamics & Dendritic-inspired architecture for action control \\
PoT-LIF standard mode & Signed power-of-two events; residual retention & Temporal sequence across layers \\
PoT-LIF Heartbeat & Signed power-of-two events; residual retention & Event generation and accumulation within each projection \\
\end{longtable}}

\section{Run-Mode Parameters and Numerical Formats}\label{app:E}

The main configurations use bfloat16 (BF16) weights; the main Heartbeat configurations use FP32 membrane state. The supplementary FP8/A24 configuration uses FP8 E4M3FN weights \citep{micikevicius2022fp8}, decoded through a BF16 container, and a 24-bit signed block-floating-point persistent accumulator. E4M3FN denotes a finite-value FP8 format with four exponent bits and three fraction bits.

{\footnotesize
\setlength{\tabcolsep}{4pt}
\renewcommand{\arraystretch}{1.1}
\begin{longtable}{@{}>{\raggedright\arraybackslash}p{\dimexpr 0.23\linewidth-1.3333333333333333\tabcolsep\relax}>{\raggedright\arraybackslash}p{\dimexpr 0.38\linewidth-1.3333333333333333\tabcolsep\relax}>{\raggedright\arraybackslash}p{\dimexpr 0.39\linewidth-1.3333333333333333\tabcolsep\relax}@{}}
\caption{Operating-mode parameters. Both modes use the integration and residual balance in Eq.~\eqref{eq:state_balance} and the emission rule in Eq.~\eqref{eq:emission}; the evaluated policies use ten external flow-matching solver steps, distinct from the internal event budget.}\label{tab:5}\\
\toprule
Parameter & Standard mode & Heartbeat \\
\midrule\endfirsthead
\multicolumn{3}{l}{\small\itshape Table \thetable\ (continued)}\\\toprule
Parameter & Standard mode & Heartbeat \\
\midrule\endhead
\midrule\multicolumn{3}{r}{\small\itshape Continued on next page}\\\endfoot\bottomrule\endlastfoot
Input flow & one current per tick & single injection ${I}_{1}=x$ \\
Decay $\lambda $ & configurable (main runs: 1) & 1 \\
Emission threshold $\theta $ & Two-sided; main runs: ${2}^{-16}$ & two-sided, code floor ${2}^{{e}_{\operatorname{min}}}={2}^{-16}$ \\
State initialization & ${r}_{0}=0$ for each image-and-prompt input & ${r}_{0}=0$ per projection invocation; ${u}_{1}=x$ \\
Residual update and reset & Subtractive update each internal step; reset at next image-and-prompt input & Subtractive update; reset after $T$ steps or early termination \\
Event application & events propagate downstream & applied inside the projection \\
Cross-layer time axis & present & none \\
\end{longtable}}

{\small
\setlength{\tabcolsep}{4pt}
\renewcommand{\arraystretch}{1.2}
\begin{longtable}{@{}>{\raggedright\arraybackslash}p{\dimexpr 0.29\linewidth-1.0\tabcolsep\relax}>{\raggedright\arraybackslash}p{\dimexpr 0.71\linewidth-1.0\tabcolsep\relax}@{}}
\caption{Training parameters and numerical formats. FP8 weights and A24 accumulation are supplementary configurations; the main table uses BF16 weights.}\label{tab:6}\\
\toprule
Parameter & Setting \\
\midrule\endfirsthead
\multicolumn{2}{l}{\small\itshape Table \thetable\ (continued)}\\\toprule
Parameter & Setting \\
\midrule\endhead
\midrule\multicolumn{2}{r}{\small\itshape Continued on next page}\\\endfoot\bottomrule\endlastfoot
Decay $\lambda $ & 1 (no leakage during local decomposition) \\
Emission threshold $\theta $ & two-sided code floor ${2}^{{e}_{\operatorname{min}}}={2}^{-16}$; extracts nonzero powers of the input significand \\
Primary Heartbeat and $T=1$ surrogate & magnitude STE on the reconstructed sum, ${g}_{T}(x)={q}_{T}(x)/x$ for $x\ne 0$ \\
Surrogate at $x=0$ & magnitude: $0$ at full strength, $1-\alpha $ during insertion; identity: $1$ \\
Standard-mode surrogate at $T=2$ & constant 0.5 including at zero membrane potential, with cross-tick backpropagation \\
FP8 weights & E4M3FN, static power-of-two per-tensor scale \\
A24 accumulator & 24-bit signed block-floating-point persistent state \\
Optimizer & Adam with decoupled weight decay (AdamW) \citep{loshchilov2019adamw}, ${\beta }_{1}=0.9$, ${\beta }_{2}=0.95$, $\epsilon ={10}^{-8}$, weight decay ${10}^{-10}$, grad-norm clip 1.0 \\
Learning-rate schedule & linear warmup 200 updates to $5\times {10}^{-6}$, cosine decay to ${10}^{-6}$; exponential moving average (EMA) decay 0.999 \\
PoT insertion & linear ramp over the first 1,000 optimizer updates ($\alpha :0\to 1$) \\
Membrane precision & FP32 in Heartbeat main runs (BF16 membrane: separate no-fine-tuning variant) \\
Seeds & training 42, evaluation 7 \\
Fine-tuning & 5,000 optimizer updates, effective batch 256 across 4 devices \\
Forward execution & SAC per nonzero event--weight pair; floating-point power-of-two scaling through exponent addition and sign handling, followed by accumulation inside the projection \\
\end{longtable}}

During gradual insertion, Heartbeat uses the forward map $x+\alpha ({q}_{T}(x)-x)$, where $\alpha =clip(s/(1000a),0,1)$, $s$ is the microstep index, and $a$ is the gradient-accumulation depth. The effective surrogate derivative is $(1-\alpha )+\alpha {g}_{T}(x)$; at zero input it is $1-\alpha $ for the magnitude surrogate. The shared $T=1$ configuration uses the magnitude STE (Appendix~\ref{app:H}). Standard mode at $T=2$ differentiates through its membrane recurrence with a local gain of 0.5. For the local recurrence with $\lambda =1$, this gives $\partial {s}_{t}/\partial {I}_{k}=0.5\cdot {0.5}^{t-k}$ for $k\le t$, with internal inputs treated as independent. Membrane states persist across the network's internal steps and reset for each new image-and-prompt input. For a nonzero weight tensor $W$, FP8 quantization uses a static power-of-two scale with exponent $e=\lceil {\operatorname{log}}_{2}(\operatorname{max}|W|/448)\rceil $ and round-to-nearest-even conversion. Biases, normalization scales, and embeddings retain their original precision. Query--key score products and attention-weighted value products remain floating-point.

\section{$T=1$ SAC Arithmetic Accounting}\label{app:F}

This appendix reports only the operation-level arithmetic energy for one dense reference product. It excludes data transport, static random-access memory (SRAM) and register access, event decoding, sign-control logic, membrane encoding, leakage, clocking, communication, and other system-level costs.

The comparison uses one conventional FP32 multiply--accumulate (MAC) as the reference. The 4.6 pJ MAC cost is decomposed into a 3.7 pJ FP32 multiplication and a 0.9 pJ FP32 accumulation.

For the PoT-LIF $T=1$ path, each input produces at most one signed power-of-two event. Using the reported 60.01\% firing rate as the effective event--weight activity fraction gives an average of 0.6001 SAC operations per reference product, with zero-event contributions skipped. Each event requires one exponent-scaling operation and one FP32 accumulation. Exponent scaling is assigned an illustrative cost of 0.1 pJ per operation, using an integer-addition cost as a proxy.

{\small
\setlength{\tabcolsep}{4pt}
\renewcommand{\arraystretch}{1.2}
\begin{longtable}{@{}>{\raggedright\arraybackslash}p{\dimexpr 0.24\linewidth-1.5\tabcolsep\relax}>{\raggedright\arraybackslash}p{\dimexpr 0.3\linewidth-1.5\tabcolsep\relax}>{\raggedright\arraybackslash}p{\dimexpr 0.3\linewidth-1.5\tabcolsep\relax}>{\raggedright\arraybackslash}p{\dimexpr 0.16\linewidth-1.5\tabcolsep\relax}@{}}
\caption{Average arithmetic energy per dense reference product with matched FP32 accumulation; the $T=1$ PoT-LIF row uses a 60.01\% firing rate. All other power terms are excluded.}\label{tab:7}\\
\toprule
Path & Multiplication / scaling (pJ) & FP32 accumulation (pJ) & Total (pJ) \\
\midrule\endfirsthead
\multicolumn{4}{l}{\small\itshape Table \thetable\ (continued)}\\\toprule
Path & Multiplication / scaling (pJ) & FP32 accumulation (pJ) & Total (pJ) \\
\midrule\endhead
\midrule\multicolumn{4}{r}{\small\itshape Continued on next page}\\\endfoot\bottomrule\endlastfoot
FP32 MAC & $1\times3.7=3.7$ & $1\times0.9=0.9$ & 4.6 \\
PoT-LIF $T=1$ & $0.6001\times0.1$ $=0.06001$ & $0.6001\times0.9$ $=0.54009$ & 0.6001 \\
\end{longtable}}

Thus, under this matched arithmetic model, PoT-LIF $T=1$ averages 0.6001~pJ per reference product and the FP32 MAC reference uses 4.6~pJ. The energy ratio is $4.6/0.6001\approx7.67$, and the arithmetic-energy reduction is $1-0.6001/4.6\approx86.95\%$.

These values are arithmetic-energy estimates rather than hardware measurements. At an operation rate of $f$ reference products per second, the corresponding modeled arithmetic powers are $4.6f$~pW for the FP32 MAC path and $0.6001f$~pW for the PoT-LIF $T=1$ path. Full-system power is outside the scope of this accounting.

\section{McNemar Pairing Statistics}\label{app:G}

All 25 comparisons use 2,000 paired episodes. For models A and B, $b$ counts A-success/B-failure pairs and $c$ counts A-failure/B-success pairs. The exact two-sided McNemar value \citep{mcnemar1947} is $p=\operatorname{min}(1,2\sum _{k=0}^{\operatorname{min}(b,c)}\binom{b+c}{k}{2}^{-(b+c)})$. Table~\ref{tab:8} groups the comparisons as follows: F1 contains the six dense-control rows, F2 the three matched-budget continuation rows, and E the 16 exploratory rows. Holm correction \citep{holm1979} is applied separately within F1 and F2. F1 includes the four main event configurations and two supplementary configurations: Heartbeat trained with the identity STE and the FP8/A24 variant. The two STE configurations are fine-tuned separately from the same initial checkpoint and evaluated using their respective final checkpoints. F2 compares standard mode at $T=1$ and Heartbeat at $T=2$ and $T=4$ against the matched-budget continuous control to account for the additional fine-tuning updates. Archived FP8/FP32 entries are exploratory; their cross-process reproducibility remains unverified. None of the four main event configurations differs significantly from the original continuous-activation control after Holm correction at the 0.05 level; this does not establish statistical equivalence.

\section{Expressive range of the event code}\label{app:H}

We separately examine the representational capacity of power-of-two increments using trainable thresholds whose positions are parameterized separately from the emitted increments. Each ordered threshold sets where an increment is triggered, while the increment is specified by its own signed power-of-two code; its magnitude is not constrained to equal the threshold. This construction is distinct from the single-threshold recurrence used in the VLA projections; its formal definition is given in Appendix~\ref{app:B}. With 32 thresholds and at most four signed power-of-two terms per increment, it approximates a clipped rectified linear unit (ReLU) scaled by 1/4, sigmoid, and hyperbolic tangent (tanh) with root mean square error (RMSE) values of 0.020, 0.033, and 0.045, respectively (Fig.~\ref{fig:3}). Threshold placement controls resolution along the input axis, while the available power-of-two terms constrain the output increments. These function-fitting results illustrate the expressive range of the event code. The ability of trainable thresholds to shape nonlinear responses motivates further investigation of this direction.

\begin{figure}[H]\centering
	\includegraphics[width=1.0\linewidth]{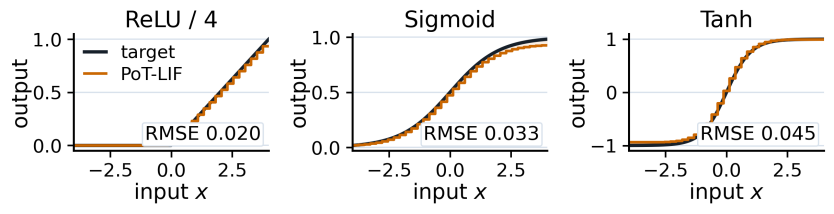}
	\caption{Function approximation with trainable threshold positions and separately parameterized power-of-two increments (Appendix~\ref{app:B}). Orange curves approximate the target functions shown in dark lines using 32 ordered thresholds and at most four signed power-of-two terms per increment. The respective RMSE values are 0.020, 0.033, and 0.045.}\label{fig:3}
\end{figure}

\clearpage
{\footnotesize
\setlength{\tabcolsep}{4pt}
\renewcommand{\arraystretch}{1.2}
\begin{longtable}{@{}>{\raggedright\arraybackslash}p{\dimexpr 0.065\linewidth-1.7777777777777777\tabcolsep\relax}>{\raggedright\arraybackslash}p{\dimexpr 0.35\linewidth-1.7777777777777777\tabcolsep\relax}>{\raggedright\arraybackslash}p{\dimexpr 0.085\linewidth-1.7777777777777777\tabcolsep\relax}>{\raggedright\arraybackslash}p{\dimexpr 0.085\linewidth-1.7777777777777777\tabcolsep\relax}>{\raggedright\arraybackslash}p{\dimexpr 0.075\linewidth-1.7777777777777777\tabcolsep\relax}>{\raggedright\arraybackslash}p{\dimexpr 0.055\linewidth-1.7777777777777777\tabcolsep\relax}>{\raggedright\arraybackslash}p{\dimexpr 0.055\linewidth-1.7777777777777777\tabcolsep\relax}>{\raggedright\arraybackslash}p{\dimexpr 0.11\linewidth-1.7777777777777777\tabcolsep\relax}>{\raggedright\arraybackslash}p{\dimexpr 0.12\linewidth-1.7777777777777777\tabcolsep\relax}@{}}
\caption{Paired statistical comparisons. A and B are success rates (\%); $\Delta =B-A$ is in percentage points. For F1 and the first nine E rows, A is the 96.70\% continuous control; for F2, A is the 96.25\% matched-budget control. Other E rows name their reference configuration. $p$ is the unadjusted exact value; $p_{\mathrm{Holm}}$ is reported only for F1 and F2. The two STE rows correspond to separately fine-tuned final checkpoints. ``BF16 membrane'' distinguishes the precision sweeps from the main configurations with FP32 membrane state.}\label{tab:8}\\
\toprule
Fam. & Configuration B / reference A & A (\%) & B (\%) & $\Delta $ & $b$ & $c$ & $p$ & $p_{\mathrm{Holm}}$ \\
\midrule\endfirsthead
\multicolumn{9}{l}{\small\itshape Table \thetable\ (continued)}\\\toprule
Fam. & Configuration B / reference A & A (\%) & B (\%) & $\Delta $ & $b$ & $c$ & $p$ & $p_{\mathrm{Holm}}$ \\
\midrule\endhead
\midrule\multicolumn{9}{r}{\small\itshape Continued on next page}\\\endfoot\bottomrule\endlastfoot
F1 & Heartbeat $T=4$, identity STE & 96.70 & 96.90 & +0.20 & 38 & 42 & 0.738 & 1.000 \\
F1 & Heartbeat $T=4$, magnitude STE & 96.70 & 96.90 & +0.20 & 38 & 42 & 0.738 & 1.000 \\
E & Heartbeat $T=4$ FP8/FP32 (archived) & 96.70 & 97.00 & +0.30 & 40 & 46 & 0.590 & --- \\
F1 & Heartbeat $T=4$ FP8/A24 & 96.70 & 97.50 & +0.80 & 38 & 54 & 0.117 & 0.704 \\
F1 & Standard mode $T=2$ & 96.70 & 96.35 & $-$0.35 & 62 & 55 & 0.579 & 1.000 \\
F1 & Heartbeat $T=2$ main & 96.70 & 97.25 & +0.55 & 40 & 51 & 0.294 & 1.000 \\
E & Heartbeat $T=2$ FP8/FP32 (archived) & 96.70 & 96.90 & +0.20 & 43 & 47 & 0.752 & --- \\
F1 & Standard mode $T=1$ BF16 & 96.70 & 96.90 & +0.20 & 44 & 48 & 0.755 & 1.000 \\
E & Heartbeat $T=2$ BF16 membrane, $\theta ={2}^{-16}$ & 96.70 & 97.15 & +0.45 & 38 & 47 & 0.386 & --- \\
E & Heartbeat $T=4$ BF16 membrane, $\theta ={2}^{-16}$ & 96.70 & 97.20 & +0.50 & 33 & 43 & 0.302 & --- \\
E & Heartbeat $T=2$ BF16 membrane, $\theta ={2}^{-8}$ & 96.70 & 97.05 & +0.35 & 43 & 50 & 0.534 & --- \\
E & Heartbeat $T=4$ BF16 membrane, $\theta ={2}^{-8}$ & 96.70 & 97.20 & +0.50 & 35 & 45 & 0.314 & --- \\
E & $T=1$ FP8 E4M3FN, $\theta ={2}^{-16}$ & 96.70 & 97.25 & +0.55 & 40 & 51 & 0.294 & --- \\
E & $T=1$ FP8 E4M3FN, $\theta ={2}^{-8}$ & 96.70 & 96.35 & $-$0.35 & 57 & 50 & 0.562 & --- \\
E & $T=1$ FP8 E4M3FN, $\theta ={2}^{-4}$ & 96.70 & 97.05 & +0.35 & 47 & 54 & 0.551 & --- \\
F2 & Standard mode $T=1$ BF16 & 96.25 & 96.90 & +0.65 & 44 & 57 & 0.232 & 0.284 \\
F2 & Heartbeat $T=2$ main & 96.25 & 97.25 & +1.00 & 36 & 56 & 0.047 & 0.141 \\
F2 & Heartbeat $T=4$ magnitude STE & 96.25 & 96.90 & +0.65 & 27 & 40 & 0.142 & 0.284 \\
E & $T=1$ FP8 / BF16 weights, $\theta ={2}^{-16}$ & 96.90 & 97.25 & +0.35 & 40 & 47 & 0.520 & --- \\
E & $T=1$ FP8 / BF16 weights, $\theta ={2}^{-8}$ & 96.90 & 96.35 & $-$0.55 & 52 & 41 & 0.300 & --- \\
E & $T=1$ FP8 / BF16 weights, $\theta ={2}^{-4}$ & 96.90 & 97.05 & +0.15 & 40 & 43 & 0.826 & --- \\
E & $T=2$ BF16 membrane / main, $\theta ={2}^{-16}$ & 97.25 & 97.15 & $-$0.10 & 40 & 38 & 0.910 & --- \\
E & $T=4$ BF16 membrane / main, $\theta ={2}^{-16}$ & 96.90 & 97.20 & +0.30 & 18 & 24 & 0.441 & --- \\
E & $T=2$ BF16 membrane / main, $\theta ={2}^{-8}$ & 97.25 & 97.05 & $-$0.20 & 39 & 35 & 0.728 & --- \\
E & $T=4$ BF16 membrane / main, $\theta ={2}^{-8}$ & 96.90 & 97.20 & +0.30 & 22 & 28 & 0.480 & --- \\
\end{longtable}}

\section{Threshold and Precision Sweeps}\label{app:I}

The two sweeps below reuse fixed checkpoints without further fine-tuning. Each configuration is evaluated on 500 episodes per suite. Event statistics are computed for each precision path separately.

Standard mode at $T=1$ with BF16 weights is the shared single-step configuration reported in Table~\ref{tab:1}. Starting from the continuous-activation checkpoint, it uses the magnitude STE and 5,000 optimizer updates at effective batch size 256. Success rates are 99.2\% on Spatial, 98.4\% on Object, 97.6\% on Goal, and 92.4\% on Long-Horizon, giving 96.90\% overall (1,938/2,000).

{\small
\setlength{\tabcolsep}{4pt}
\renewcommand{\arraystretch}{1.2}
\begin{longtable}{@{}>{\raggedright\arraybackslash}p{\dimexpr 0.06\linewidth-1.75\tabcolsep\relax}>{\raggedright\arraybackslash}p{\dimexpr 0.11\linewidth-1.75\tabcolsep\relax}>{\raggedright\arraybackslash}p{\dimexpr 0.14\linewidth-1.75\tabcolsep\relax}>{\raggedright\arraybackslash}p{\dimexpr 0.14\linewidth-1.75\tabcolsep\relax}>{\raggedright\arraybackslash}p{\dimexpr 0.13\linewidth-1.75\tabcolsep\relax}>{\raggedright\arraybackslash}p{\dimexpr 0.14\linewidth-1.75\tabcolsep\relax}>{\raggedright\arraybackslash}p{\dimexpr 0.15\linewidth-1.75\tabcolsep\relax}>{\raggedright\arraybackslash}p{\dimexpr 0.13\linewidth-1.75\tabcolsep\relax}@{}}
\caption{Heartbeat with BF16 membrane state and explicit thresholds, using fixed checkpoints and BF16 weights without further fine-tuning. The implementation's four-bit threshold-selection index maps $c=15$ to $\theta ={2}^{-16}$ and $c=7$ to $\theta ={2}^{-8}$; it is not the event-exponent field. $\Delta $ is relative to the main Heartbeat configuration at the same $T$, which uses FP32 membrane state and $\theta ={2}^{-16}$. Rates are percentages; differences are percentage points.}\label{tab:9}\\
\toprule
$T$ & $\theta $ & Spatial & Object & Goal & Long-H. & Overall & $\Delta $ \\
\midrule\endfirsthead
\multicolumn{8}{l}{\small\itshape Table \thetable\ (continued)}\\\toprule
$T$ & $\theta $ & Spatial & Object & Goal & Long-H. & Overall & $\Delta $ \\
\midrule\endhead
\midrule\multicolumn{8}{r}{\small\itshape Continued on next page}\\\endfoot\bottomrule\endlastfoot
2 & ${2}^{-16}$ & 99.4 & 98.8 & 97.6 & 92.8 & 97.15 & $-$0.10 \\
2 & ${2}^{-8}$ & 98.2 & 98.6 & 98.4 & 93.0 & 97.05 & $-$0.20 \\
4 & ${2}^{-16}$ & 99.4 & 97.6 & 98.8 & 93.0 & 97.20 & $+$0.30 \\
4 & ${2}^{-8}$ & 98.8 & 98.2 & 99.0 & 92.8 & 97.20 & $+$0.30 \\
\end{longtable}}

{\small
\setlength{\tabcolsep}{4pt}
\renewcommand{\arraystretch}{1.2}
\begin{longtable}{@{}>{\raggedright\arraybackslash}p{\dimexpr 0.14\linewidth-1.7142857142857142\tabcolsep\relax}>{\raggedright\arraybackslash}p{\dimexpr 0.14\linewidth-1.7142857142857142\tabcolsep\relax}>{\raggedright\arraybackslash}p{\dimexpr 0.14\linewidth-1.7142857142857142\tabcolsep\relax}>{\raggedright\arraybackslash}p{\dimexpr 0.13\linewidth-1.7142857142857142\tabcolsep\relax}>{\raggedright\arraybackslash}p{\dimexpr 0.14\linewidth-1.7142857142857142\tabcolsep\relax}>{\raggedright\arraybackslash}p{\dimexpr 0.16\linewidth-1.7142857142857142\tabcolsep\relax}>{\raggedright\arraybackslash}p{\dimexpr 0.15\linewidth-1.7142857142857142\tabcolsep\relax}@{}}
\caption{Standard mode at $T=1$ with FP8 E4M3FN weights and three explicit thresholds, using the final $T=1$ checkpoint without further fine-tuning and a static power-of-two scale per weight tensor. $\Delta $ is relative to the 96.90\% BF16-weight $T=1$ result. Rates are percentages; differences are percentage points.}\label{tab:10}\\
\toprule
$\theta $ & Spatial & Object & Goal & Long-H. & Overall & $\Delta $ \\
\midrule\endfirsthead
\multicolumn{7}{l}{\small\itshape Table \thetable\ (continued)}\\\toprule
$\theta $ & Spatial & Object & Goal & Long-H. & Overall & $\Delta $ \\
\midrule\endhead
\midrule\multicolumn{7}{r}{\small\itshape Continued on next page}\\\endfoot\bottomrule\endlastfoot
${2}^{-16}$ & 98.8 & 98.6 & 98.0 & 93.6 & 97.25 & $+$0.35 \\
${2}^{-8}$ & 98.2 & 98.0 & 97.6 & 91.6 & 96.35 & $-$0.55 \\
${2}^{-4}$ & 98.6 & 97.6 & 97.8 & 94.2 & 97.05 & $+$0.15 \\
\end{longtable}}

The BF16-membrane sweep changes overall success by $-$0.20 to +0.30 percentage points relative to the corresponding main Heartbeat configuration. The FP8-weight sweep changes success by $-$0.55 to +0.35 points relative to the BF16-weight single-step configuration. These fixed-checkpoint results describe sensitivity to the tested thresholds and numerical formats. The sweep comparisons are exploratory and fall outside the two corrected families in Appendix~\ref{app:G}.

\clearpage

\setlength{\bibsep}{5pt}
\interlinepenalty=10000
\bibliographystyle{unsrtnat}

\bibliography{references}

\end{document}